\documentclass[lettersize,journal]{IEEEtran}

\IEEEoverridecommandlockouts                              

\usepackage{graphics} 
\usepackage{epsfig} 

\usepackage{amsmath} 
\usepackage{amssymb}  
\usepackage{mathtools} 
\usepackage{bm} 
 \usepackage{multirow}
 \usepackage[linesnumbered,ruled,vlined,Algorithm ]{algorithm }

\usepackage{algorithmicx}
\usepackage{algpseudocode}
\usepackage{hyperref}
\usepackage{array}
\usepackage{caption}

\usepackage{subcaption}

\usepackage{enumerate}

\usepackage[english]{babel} 
\usepackage{amsfonts,amsthm,bm} 
 \usepackage{cleveref}
\newtheorem{theorem}{Theorem} 
\newtheorem{corollary}{Corollary}[theorem]

\usepackage{graphicx}
\usepackage{wrapfig }
\usepackage{multirow}
 \usepackage[edges]{forest}

\usepackage{tcolorbox}
\usepackage{tabularx,colortbl}
\usepackage{adjustbox}
\usepackage{colortbl}
\usepackage{hhline}
 
\graphicspath{{Images/}}

\usepackage{cite}

\title{\LARGE \bf A Non-iterative Spatio-temporal Multi-task Assignments based Collision-free Trajectories for  Music Playing Robots}

\author{Shridhar Velhal$^{1}$, Krishna Kishore VS$^{2}$ and Suresh Sundaram$^{3}$ 
\thanks{$^{1}$Shridhar Velhal is a Ph.D. student in Department of Aerospace Engineering, Indian institute of Science, Bangalore, India. 
        {\tt\small velhalb@iisc.ac.in}}%
\thanks{$^{2}$Krishna Kishore VS is a Third year B.Tech student  in Department of Electronics and Communication Engineering, National Institute of Technology, Tiruchirappalli, India. 
        {\tt\small  krishnakishorevs@gmail.com}}%
\thanks{$^{3}$ Suresh Sundaram is an Associate Professor in Department of Aerospace Engineering, Indian institute of Science, Bangalore, India.         {\tt\small vssuresh@iisc.ac.in}}%
}

\begin{document}
\maketitle
\thispagestyle{empty}
\pagestyle{empty}

\begin{abstract}
In this paper, a non-iterative spatio-temporal multi-task assignment approach is used for playing piano music by a team of robots. This paper considers the piano playing problem, in which an algorithm needs to compute the trajectories for a dynamically sized team of robots who will play the musical notes by traveling through the specific locations associated with musical notes at their respective specific times. A two-step dynamic resource allocation based on a spatio-temporal multi-task assignment problem  (DREAM), has been implemented to assign robots for playing the musical tune. The algorithm computes the required number of robots to play the music in the first step. In the second step, optimal assignments are computed for the updated team of robots, which minimizes the total distance traveled by the team. 
Even for the individual feasible trajectories, the multi-robot execution may fail if robots encounter a collision. As some time will be utilized for this conflict resolution, robots may not be able to reach the desired location on time. This paper analyses and proves that, if robots are operating in a convex region, the solution of the DREAM approach provides collision-free trajectories.
The working of the DREAM approach has been illustrated with the help of the high fidelity simulations in Gazebo operated using ROS2. The result clearly shows that the DREAM approach computes the required number of robots and assigns multiple tasks to robots in at most two steps. The simulation of the robots playing music, using computed assignments, is demonstrated in the attached video.
video link: \url{https://youtu.be/XToicNm-CO8}
\end{abstract}
 
\section{Introduction}
Multi-robot systems have been used to execute tasks where a team of robots needs to visit a set of spatially distributed locations. The robots have been used for different purposes such as logistics \cite{chen2021integrated}, surveillance  \cite{notomista2019optimal,harikumar2019mission,venugopalan2015multi} and various other applications \cite{khamis2015multi,korsah2013comprehensive,konstantakopoulos2022vehicle}. Robots are assigned to tasks such that they minimize some specific criteria (e.g., minimization of total distance traveled \cite{cheikhrouhou2021comprehensive}, completion time \cite{lee2018resource,xu2021completion}, 
energy consumption \cite{luan2020energy}). With recent developments in IoT and communication infrastructure, customers demand spatial visits with time constraints. In logistics operations, the delivery person not only needs to visit the locations but needs to visit those locations with some time constraints. The multi-task assignment problems with different temporal constraints are reviewed in \cite{nunes2017taxonomy}.

\begin{figure}[t!]
    \centering
    \includegraphics[width=0.90\linewidth]{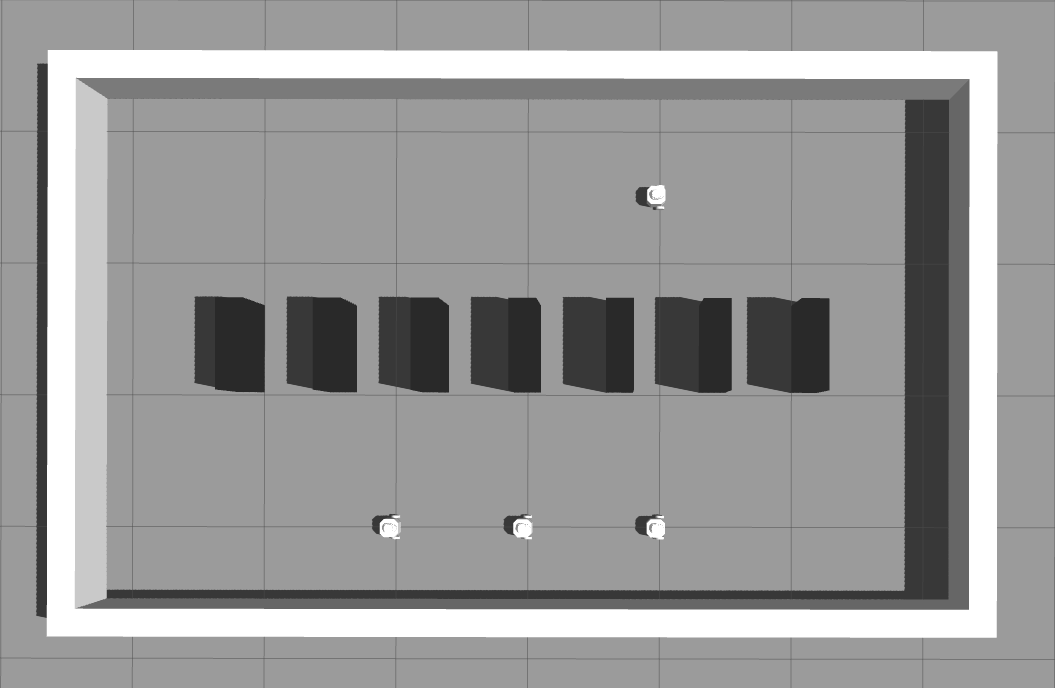}
    \caption{Piano Setup in Gazebo for playing music using turtlebots }
    \label{fig:set-up}
\end{figure}
 
The problem of music-playing robots is formulated by Smriti Chopra and Magnus Egerstedt in  \cite{chopra2012multi,chopra2014heterogeneous,chopra2015spatio}.
For playing a musical note, a robot must visit the specific location which plays that sound note. Each musical note is connected to a different spatial location. Robots must visit the prescribed note locations at specific times to play musical tunes. In this way, the problem of music-playing robots is converted to multiple spatio-temporal tasks for the team of robots. A team of robots needs to compute the assignments of robots to multiple spatio-temporal tasks, following which they can play the music.

The task allocation and scheduling schemes \cite{zhang2013multi,bischoff2020multi} have been used to schedule the tasks for multiple robots. But in these approaches tasks are handled with soft temporal constraints and all tasks will not be executed at the desired time. Hence can not be applied for the exact time constraints. Also due to the exact-time constraints, time-windowed task assignment approaches can not be used.  The just-in-time tasks assignment \cite{monden2011toyota,emde2012optimally,nishida2019just,nishida2022dynamic} approaches have been employed in automobile and automation industries, but it assumes tasks are predefined and the solution requires the off-line computations. The spatio-temporal multi-task assignment (STMTA) problem has been studied in \cite{chopra2012multi,chopra2014heterogeneous,chopra2015spatio , velhal2021decentralized,velhal2022dynamic}.  
 
For the given spatio-temporal tasks, depending on the tasks and velocity of robots, some minimum number of robots are required to execute them. If the formulated STMTA is ill-posed, then the obtained solution is not valid and demands unrealistic velocities. The velocity-constrained STMTA problem is solved in \cite{chopra2014heterogeneous}. A Hungarian problem \cite{kuhn2005hungarian} is solved multiple times, by increasing the robots one by one until a feasible solution is achieved. It is an iterative method and can not be used for online computations.

Recently, a non-iterative solution has been proposed using Dynamic resource allocation with a decentralized multi-task assignment (DREAM) approach \cite{velhal2022dynamic}. DREAM solves the STMTA (Hungarian) problem, once to compute the required minimum number of robots. Then, with an updated team of robots, STMTA is solved to get the optimal solution. So with only two steps, DREAM computes the required number of robots and their optimal assignments. 

In this paper, the music-playing problem is solved and the execution guarantee is proved.  Once the spatio-temporal tasks are generated for a given music, the current team of robots uses the DREAM algorithm. It computes the minimum number of robots required to play musical notes. Required robots are added to the team, and the updated STMTA problem is solved to get optimal assignments. Once the assignments are computed, the trajectory generation algorithm is used to compute the feasible trajectories for each robot. For multi-robot systems, even if all individual trajectories are feasible, one can not guarantee execution due to the possibility of collisions among robots. Also note that reactive collision avoidance methods require some extra (and uncertain) time to resolve the conflicts. As some time will be utilized for conflict resolution, the temporal feasibility of tasks may become invalid. Robots may not be able to reach the desired location at the desired time. For the playing music, the untimely play of musical notes leads to dissonance. Hence, there is a need to guarantee that the solution is feasible and has collision-free trajectories.

The main contribution of this paper is the theoretical analysis of the collision-free nature of obtained trajectories. In this paper, we prove that, for homogeneous robots operating in convex space, the obtained trajectories are collision-free. By following the computed trajectories, robots will not collide with other teammates and unassigned robots (which were removed from the team). The same has been proved for robots operating in the arena (simulated environment) used in the paper. 
The working of the music-playing robots is demonstrated in Gazebo simulations, with turtlebots operated using ROS2 (galactic). The high-fidelity simulation demonstrates the online computability of the non-iterative DREAM algorithm in practical use.

The rest of the paper is organized as follows: Section \ref{sec:problem_formulation} provides the mathematical problem formulation for piano-playing robots. Section \ref{sec:STMTA} explains the DREAM approach to solve the spatio-temporal tasks and proves the collision-free nature of obtained trajectories. Section \ref{sec:Simulations} explain the simulation architecture used in ROS2-Gazebo to operate robots to play music and presents the performance of robots for playing music. Finally, Section \ref{sec:Conclusion} concludes the paper.

\section{Problem Formulation} \label{sec:problem_formulation}
Consider a two-dimensional arena with walls and musical strings as shown in fig \ref{fig:set-up}. Consider a string placed at the center of lanes between the walls, which will generate sounds with a specific frequency whenever the robots cross that midsection of the lane. 
Robots will travel through the lanes and plunk the strings to play the musical notes. With this piano set-up, a musical tune is interpreted as the sequence of the musical notes to be played at specific time instants. Robots need to reach the respective note locations at specific times. This defines the spatio-temporal tasks for robots. Robots can play the music by routing through such timed positions. Next, tasks and operations for piano music-playing robots are mathematically formulated. Before that, we list the notations used in the paper.

 The following symbols are used in paper,\\
$R$ :  robot  \\
$T$ : task  \\
$t $ : time of   task  \\
${\bm{p}}  $ : position   \\
$\mu$ : sequence of tasks assigned to robot  \\
$(\cdot)^R_i $: for ith robot \\
$(\cdot)^T_j $ : for jth task \\
$t_j$ : desired execution time of task $T_j$ \\
$T_j({\bm{p}}_j^T,t_j  ) =  T_j :  $ $j^{th}$ spatio-temporal task \\
$C_{ij}^f$ = cost of $R_i$ executing the  $T_j$ as a first task \\
$C_{kj}^s$ = cost of robot will execute the  task $T_j$ just after the task $T_k$ (subsequent task).\\
$\delta_{ij}^f$ : decision variable whether robot $R_i$ execute the task $T_j$ as first task or not.\\
$\delta_{kj}^s$ : decision variable whether robot  execute the task $T_j$ just after task $T_k$ or not.\\

Fig \ref{fig:set-up} shows the piano arena created in the Gazebo 11, and Turtlebot3 are used as robots.  The arena consists of 7 equally separated walls. The space between the separated walls is the piano tiles or keys. Each tile is associated with a predefined specific music tune. A robot must cross the lane's midsection to play that specific note. A robot needs to wait at the entry point till the desired time; it must cross the lane at the desired time. The musical note will be played once the robot passes through that lane's midsection.

\begin{figure}[b!]
    \centering
    \includegraphics[width=0.95\linewidth]{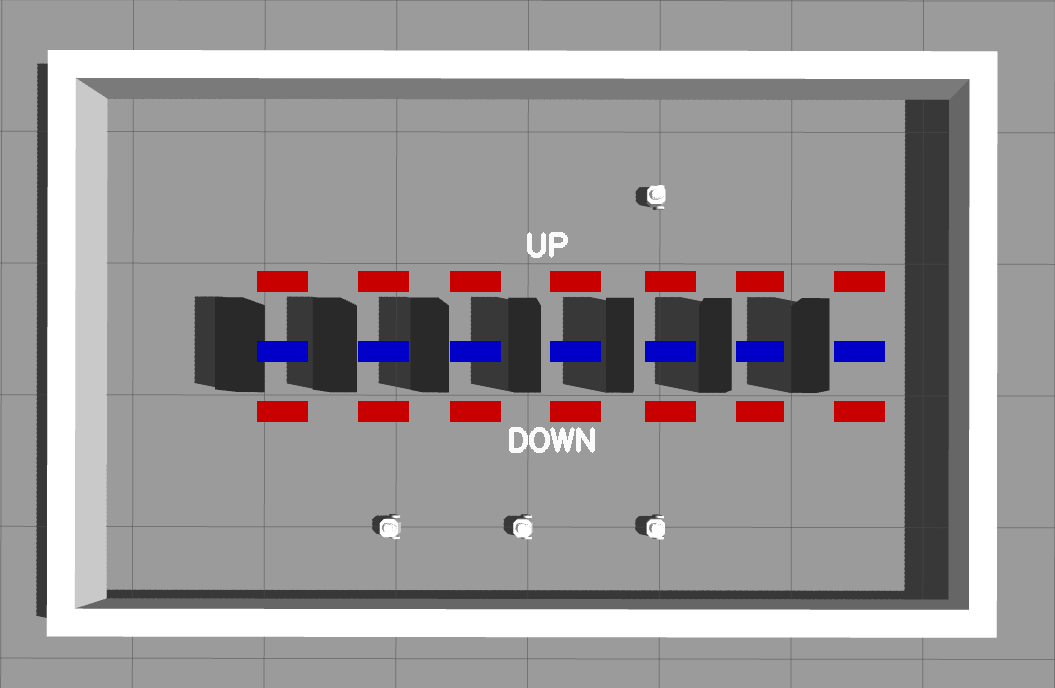}
    \caption{Illustration of notes of piano and playing of notes}
    \label{fig:lanes}
\end{figure}

Fig \ref{fig:lanes} shows the piano setup highlighting the lanes (traveling through which musical note is played) and waiting area. 
The musical notes are played whenever the robot crosses the middle patches in the lane. A robot can cross the patch either by traveling top-to-bottom or bottom-to-top.
Now robot needs to play the desired musical notes at specific times. Hence, they will wait until the desired time in a waiting area at the entry of lanes (from top or bottom). This waiting area is marked by red colored rectangles on both lane entries.

If the robot is coming from the top, it waits at the top waiting area (if it comes from below, it waits at the bottom). The robot will wait till $t_\star - \tau$, where $t_\star$ is the desired time at which the task is to be executed. $\tau$ is the buffer time for the robot to travel from the waiting area to the mid-point of the lane. The robot will not wait at mid-point; instead, it crosses the midpoint and reaches another end of the lane. This way, the musical note is played only once. Also, a buffer time of 0.05 sec is added to the musical note to avoid playing the note multiple times when robots travel through the lane.

\subsection{Mathematical Problem Formulation}
Consider  a  set of $N$ robots denoted as  $\mathcal{R}$,  $ \mathcal{R} =  \{R_1,R_2,\cdots,R_N\}$. The positions of robots $R_i$ is denoted as ${\bm{p}}_i^R = (x_i^R , y_i^R)$. 
Robots needs  to play the bunch of piano notes in properly timed sequences.

\subsubsection{Spatio-temporal tasks }
Consider a musical note $j$  with the position   ${\bm{p}}_j^T = (x_j^T, y_j^T)$ and desired time $t_j$.
Consider a musical note $j$  with the position   ${\bm{p}}_j^T = (x_j^T, y_j^T)$ and desired time $t_j$. To play this note, the robot has to visit a waiting area near the desired lane, wait until the desired execution time, and then pass through the desired location ${\bm{p}}_j^T$  at desired time $t_j$ to play that musical note.  The steps in playing notes can be listed as
\begin{enumerate}[(a)]
    \item reach to the waiting  location near  ${\bm{p}}_j^T$
    \item wait till desired time  $t_j$
    \item cross the lane
\end{enumerate}

In rest of the paper, the musical note playing is defined as a spatio-temporal task $T_j({\bm{p}}_j^T,t_j)$ is referred as $T_j$.  
This paper assumes that all robots are homogeneous. Each robot has the same maximum velocity ($v^{max}$), and can play all musical notes. The robots can communicate and share their pose with the team. The assignments are computed in a centralized way, and each robot will be communicated to execute its respective tasks.

To play the musical tune, a team of robots needs to play a sequence of musical notes. Let us consider that a musical tune has $M$ notes (in general $M>N$), and a team of robots needs to execute the given $M$ spatio-temporal task. This defines the spatio-temporal multi-task assignment problem.
The objective of the problem is to compute the assignments for the team of robots such that they will play the musical notes at respective times. One should note that for playing the musical notes, i.e., for executing the spatio-temporal tasks, some minimum number of robots is required. If the given number of robots is less than this required number, then the problem is ill-posed, and all tasks can not be executed. Hence the objective of the problem is twofold; the first is to compute the minimum number of robots required to play the given musical notes, and the second is to compute the assignments for an updated team of robots to play the given musical notes.

\section{Spatio-Temporal Multi-Task Assignments  }\label{sec:STMTA}
The team of robots needs to play the musical notes by traveling through the lane of respective notes at desired times. These task demands the robots to reach specific spatial locations at specific times; hence tasks are termed spatio-temporal tasks. For executing the given spatio-temporal tasks, some minimum number of robots is required. One needs to compute this minimum number of robots and use at least those many robots to execute the tasks. Otherwise, the problem is ill-posed, and the given team of robots will not be able to execute all the spatio-temporal tasks with any assignments. This minimum number of robots required depends on the maximum velocity of robots and the given spatio-temporal tasks.

For the computation of feasible assignments, we use the DREAM approach proposed in \cite{velhal2022dynamic} to assign the spatio-temporal tasks.
The DREAM approach uses a two-step method in which, at first, the infeasible tasks are assigned a cost equal to a large value, and the task assignment problem is solved with a given number of robots. Now from the solution, the number of infeasible assignments is identified, and those many robots are added to the team of robots. The updated team size is the minimum number of robots, and the proof of this optimality is provided in \cite{velhal2022dynamic}.
At first,  the cost function is defined for the tasks. Next, the spatio-temporal multi-task assignment problem is defined, and finally, the DREAM algorithm is provided to compute the minimum number of robots required and optimal assignments.

\subsection{Cost Function}
The cost of a spatio-temporal task is the distance that needs to be traveled by a robot to reach the task location on or before the time of the task from its previous location. For a robot executing its first task from its initial position, the cost of the first task ($C^f$) is the distance traveled by the robot from its current position to the task position on or before the desired task time.

First we denote $d({\bm{p}}_i^R,{\bm{p}}_j^T)$  as the distance computed along the feasible path from point ${\bm{p}}_i^R $  to  ${\bm{p}}_j^T $. This distance is computed using the $A^*$ motion planning algorithm on the map of the arena. Now the cost for first  task is defined as
 \begin{align}
C^f_{ij} &= \begin{cases}      d({\bm{p}}_i^R,{\bm{p}}_j^T)      & \text{if } \dfrac{ d({\bm{p}}_i^R,{\bm{p}}_j^T ) }{V_{max}^R}  \le  t_j   \\  
\kappa &  \text{ otherwise}
\end{cases} \label{eq:cost_DT}    \\  
   &\text{for } i \in \mathcal{I} = \{1,2,\cdots,N\} ,\quad   j \in \mathcal{J} = \{1,2,\cdots,M \} \qquad   \nonumber 
\end{align}
where, $\kappa$ is a large value.

The cost for executing the subsequent tasks ($C^s$)  by the robot is the distance traveled by the robot from its previous task location to reach the current task location on or before the desired execution time of that subsequent task.
\begin{align}
  C^s_{kj} &= \begin{cases}  d({\bm{p}}_k^T,{\bm{p}}_j^T)  & \text{if } t_{k,j}^{min} \le ( t_j -t_k   )   \\
\kappa  & \text{if }  t_{  k,j}^{min} >    ( t_j -t_k ) >0 \\
\infty & \text{if }    t_j - t_k    \le 0
\end{cases}  \label{eq:cost_TT1}      \\ 
 &\text{for } k \in \mathcal{K} = \{ 1,2,\cdots,M-1 \} \quad   j \in \mathcal{J}   \nonumber    
\end{align}

where $ t_{  k,j}^{min} $ is the minimum time required by robot to travel from the location of task $T_k$ to task $T_j$ and it is computed as
\begin{align}
t_{  k,j}^{min}  = \dfrac{  d({\bm{p}}_k^T,{\bm{p}}_j^T)   }{V_{max}^R}  
\end{align}

One should note that the distances computed here are with the central point in the lane. In execution, robots will wait at the waiting location near that lane and travel that lane and go to the next waiting location. To execute the same note, a robot needs to travel the lane once; hence the distance is equal to the length of the lane. For moving to another lane, the actual distance computed on the map for center points of lanes is the same as the distance from one waiting point to another waiting point.

\subsection{Optimization Problem} 
A task allocation algorithm assigns robots to multiple spatio-temporal tasks. A robot will execute the tasks in a sequence, and we denote the sequence assigned to robot $R_i$ by $\mu_i$. Here, the problem of computing sequence $\mu_i$ has been converted to compute each move of one robot from one location to another; combining all moves, one can get the sequence of tasks. Each robot computes its sequence to execute all spatio-temporal tasks while minimizing the distance traveled. The first decision variable $\delta^f_{ij}$ is used to denote that either a robot moves from position $p^R_i$ to execute the task $T_j$ by reaching  ${\bm{p}}^T_j$ on or before time $t_j$ . The subsequent decision variable   $\delta^s_{kj}$ is used to denote that either a robot moves from its previous task's position $p^T_k$ (at time  $t_k$)   to execute the next task $T_j$ at location ${\bm{p}}^T_j$  on or before the desired time  $t_j$. 
The integer programming problem is defined as,
\begin{subequations} 
  \addtocounter{equation}{-1}
\begin{align} \label{eq:integer_prog}
     \min_{\delta^f_{ij} \ \delta^s_{kj}}  & \sum_{i\in \mathcal{I} }  \sum_{j  \in  \mathcal{J}    }  C^f_{ij} \delta^f_{ij}  +  \sum_{k\in \mathcal{K} }    \sum_{j  \in   \mathcal{J}    }  C^s_{kj} \delta^s_{kj}  \ \  \\ 
 {\rm s. \ t.} \   
 & \delta^f_{ij} \in \{0,1\}\qquad \forall (i,j) \in {  \mathcal{I} \times {\mathcal{J}  } } \label{eq:cost_cond_1a} \\
 & \delta^s_{kj} \in \{0,1\}\qquad \forall (k,j) \in {  \mathcal{K}   \times {\mathcal{J} } } \label{eq:cost_cond_1b} \\
 &\sum_{i \in \mathcal{I}}  \delta^f_{ij} + \sum_{k \in {\mathcal{K}  } }  \delta^s_{kj} = 1  \quad  \forall j \in  {\mathcal J} \label{eq:cost_cond_2} \\ 
 &\sum_{j \in {\mathcal{J} }}  \delta^f_{ij}   \le 1    \quad  \forall  i \in   \mathcal{I}     \label{eq:cost_cond_3}  \\
 &  \sum_{j \in {\mathcal{J} }}  \delta^s_{kj} \le 1   \quad  \forall   k \in  {\mathcal{K}}       \label{eq:cost_cond_4}  
\end{align}  
\end{subequations}

All tasks must be assigned as a first or subsequent task to exactly one robot; this constraint is given by \eqref{eq:cost_cond_2}. A robot can move to at most one task location just after completing the current task, which is constrained by eq. \eqref{eq:cost_cond_3} and  \eqref{eq:cost_cond_4}. This optimization problem is solved using the 'linear\_sum\_assignment' function from optimize package of the SciPy library, based on \cite{crouse2016implementing}.

\begin{algorithm}[bp!] 
\caption{ Dynamic resource  allocation  algorithm } \label{algo:DREAM}
\begin{algorithmic}[1]
\State  {Initialize with  $N$ robots and $M$ musical notes to be played with respective timings }
\State {Solve the optimization problem (eq. \eqref{eq:integer_prog}) } \label{algo:label:IP}
\State {$q=$ number of assignments with cost equal to $\kappa$}
\If {   $q>0$   }
\State Add new $q$ robots in team near to the spatial location of infeasible tasks  
\State $N = N+q $
\State {Solve the optimization problem (eq. \eqref{eq:integer_prog}) }
\EndIf 
\State{feasible assignments are obtained}
\State {Compute the sequence of tasks assigned to each robots using the assignments} \label{algo:label:stop}
\end{algorithmic}
\end{algorithm}

\subsection{DREAM Approach} 
The spatio-temporal tasks require some minimum number of robots to execute the given spatio-temporal tasks. The problem formulated in Eq.~\eqref{eq:integer_prog} solves the assignment problem for a given $N$ number of robots. As the infeasible tasks are given a cost equal to the $\kappa$, the solution of Eq ~\eqref{eq:integer_prog} may contain some infeasible assignments with a cost equal to $\kappa$. The DREAM approach computes the minimum number of robots required to execute all the given tasks (for which all computed assignments will be feasible). For clarity, the DREAM algorithm \cite{velhal2022dynamic} is provided briefly.
 
\subsubsection{Trajectory generation}
Once all the feasible assignments are obtained using the 2-step solution provided by the DREAM approach, robots must execute the assigned tasks in sequence. The assignment solution provides the assignments of robots from one location to another. Augmenting the next positions to the last, one can compute the sequence of tasks assigned to each robot. The detailed trajectory generation algorithm can be found in \cite{velhal2022dynamic}.

\subsection{Analysis for collision-free trajectories}
Analysis in \cite{velhal2022dynamic} proved that the solution obtained from DREAM is optimal, and all trajectories are feasible for individual robots. But in multi-robot systems, the feasibility guarantee may fail if robots encounter any collision.
If all trajectories are collision-free, execution can be guaranteed.

\paragraph*{\textbf{Remark}} One should note that there is a subtle difference between the optimal solution of the multiple travelling salesman problem(mTSP) and STMTA. In mTSP, no two optimal paths cross each other. In STMTA, due to temporal-feasibility constraints, optimal paths may cross each other, but their trajectories will be collision-free.

\begin{theorem}
For homogeneous robots operating in a convex region, the feasible solution obtained from the DREAM algorithm gives collision-free trajectories. \label{theorem:2}
\end{theorem}

\begin{proof} 
We will prove the theorem by contradiction. Without the loss of generality, we assume that the computed trajectories of robot $R_1$ and $R_2$ have conflict and will collide. Let us consider $R_1$ is moving from ($\bm{p}_{A} , t_{A}$)  towards spatio-temporal point  ($\bm{p}_{B} , t_{B}$) with velocity $v_1$ and  $R_2$ is moving from ($\bm{p}_{C} , t_{C}$)  towards spatio-temporal point ($\bm{p}_{D}, t_{D}$) with  velocity  $v_2$   and while doing this the robots will collide at location $\bm{p}_{\ast}$ at time $t_{\ast}$. Without loss of generality, we consider $v_1 \le v_2 \le v^{max} $. This scenario is shown in Fig \ref{fig:collison_free} for clarity. 

\begin{figure}[tbhp!]
    \centering
    \vspace{5pt}
    \includegraphics[width=0.6\linewidth]{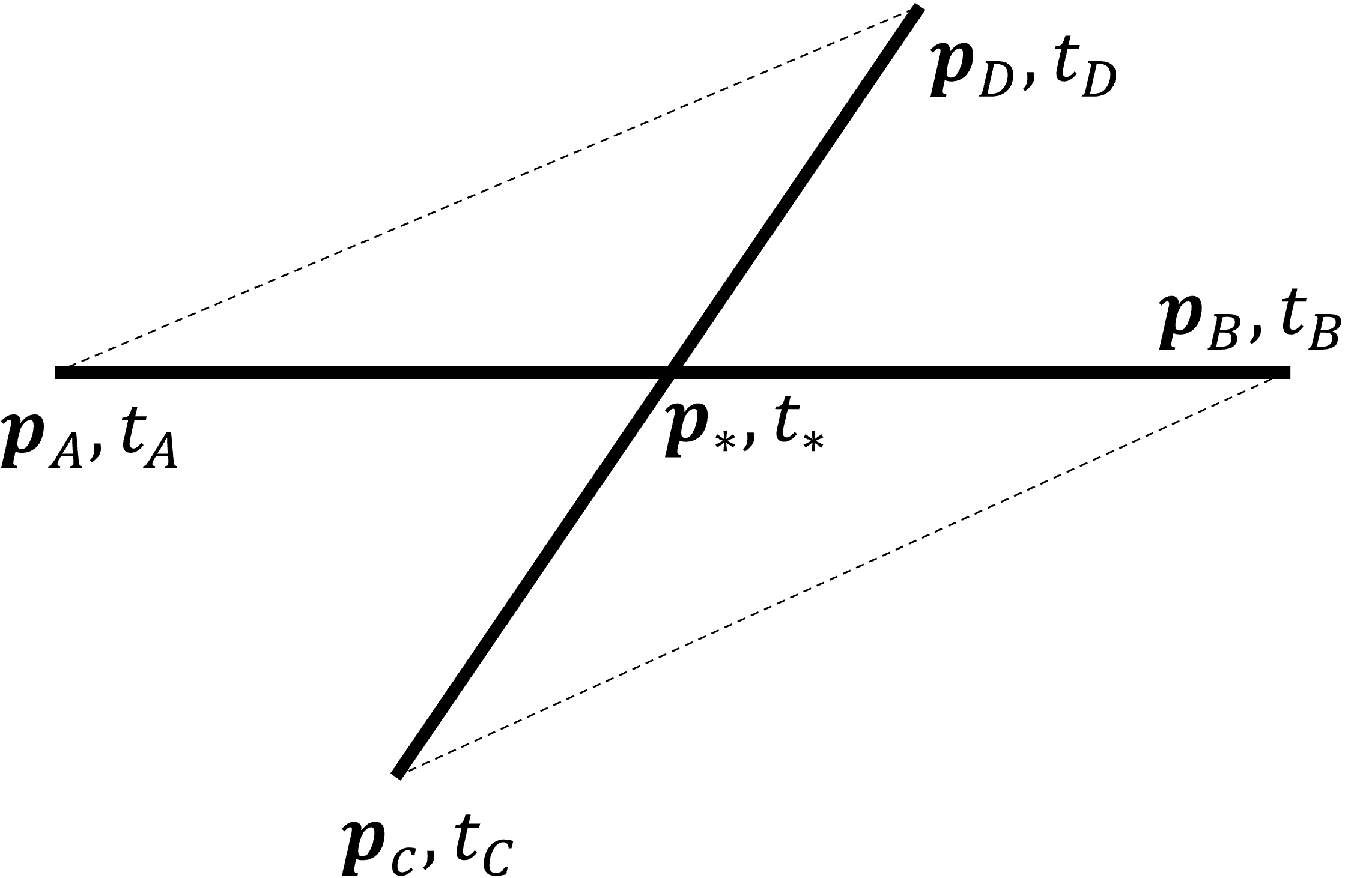}
    \caption{Geometric illustration for proof of Theorem \ref{theorem:2}}
    \label{fig:collison_free}
\end{figure}

As the solution is optimal, $ {\| \bm{p}_{A}\bm{p}_{B}  \|}_2 + {\| \bm{p}_{C}\bm{p}_{D} \|}_2 $ is the minimum cost for executing tasks $T_B$, and $T_D$ by robots $R_1$ and $R_2$

Now due to collision both robot reach $\bm{p}_{\ast}$ at time $t_{\ast}$. From the collision points, the remaining trajectories will be feasible for both robots. As robots are homogeneous, if they exchange their trajectories, the trajectories still remain feasible. (Let us refer to them as alternate trajectories). 

In the alternate solution, the tasks after collision points do not change the total distance that needs to be traveled by both robots. Alternate trajectories are feasible and require traveling the same distance; hence they are also optimal. i.e. 
$ {\| \bm{p}_{A}\bm{p}_{\ast}  \|}_2  + {\| \bm{p}_{\ast}\bm{p}_{D}  \|}_2+ {\| \bm{p}_{C}\bm{p}_{\ast} \|}_2   + {\| \bm{p}_{\ast}\bm{p}_{B}  \|}_2$ is the minimum cost for executing tasks $T_B$, and $T_D$ by robots $R_1$ and $R_2$
 
In alternate optimal solution $R_1$ travels along the path  $  \bm{p}_{A} - \bm{p}_{\ast} - \bm{p}_{D} $, but there exists direct path $ \bm{p}_{A} - \bm{p}_{D}$. By the triangle inequality, direct path from $\bm{p}_{A}$ to $\bm{p}_{D}$ is shorter. On direct path $\bm{p}_{A} - \bm{p}_{D}$,  robot $R_1$ can reach location $\bm{p}_{D}$ before the desired time $t_D$. i.e., the direct path is feasible. Hence the cost of the alternate solution is not optimal. 
 
This contradiction proves that the assumption of the existence of a collision point is incorrect. Thus it is proved that the optimal solution obtained from the DREAM approach gives collision-free trajectories.
\end{proof}

\begin{corollary}
For homogeneous robots operating in the convex region, if the robot follows the trajectory computed by the DREAM algorithm, then the robot will not collide with any unassigned robots. \label{corollary:2}
\end{corollary}

The proof of Corollary \ref{corollary:2} is straightforward. It can be proven with a case of $v_1 = 0$ (unassigned defender executes the fictitious task at the same location and travels with zero velocity) in the proof given for Theorem \ref{theorem:2}.

 \begin{corollary}
For the music-playing problem defined in this paper, the computed trajectory of robots is collision-free. \label{corollary:3}
\end{corollary}

\begin{proof}
\begin{figure}[bht!]
    \centering
      \vspace{5pt}
    \includegraphics[width=0.70\linewidth]{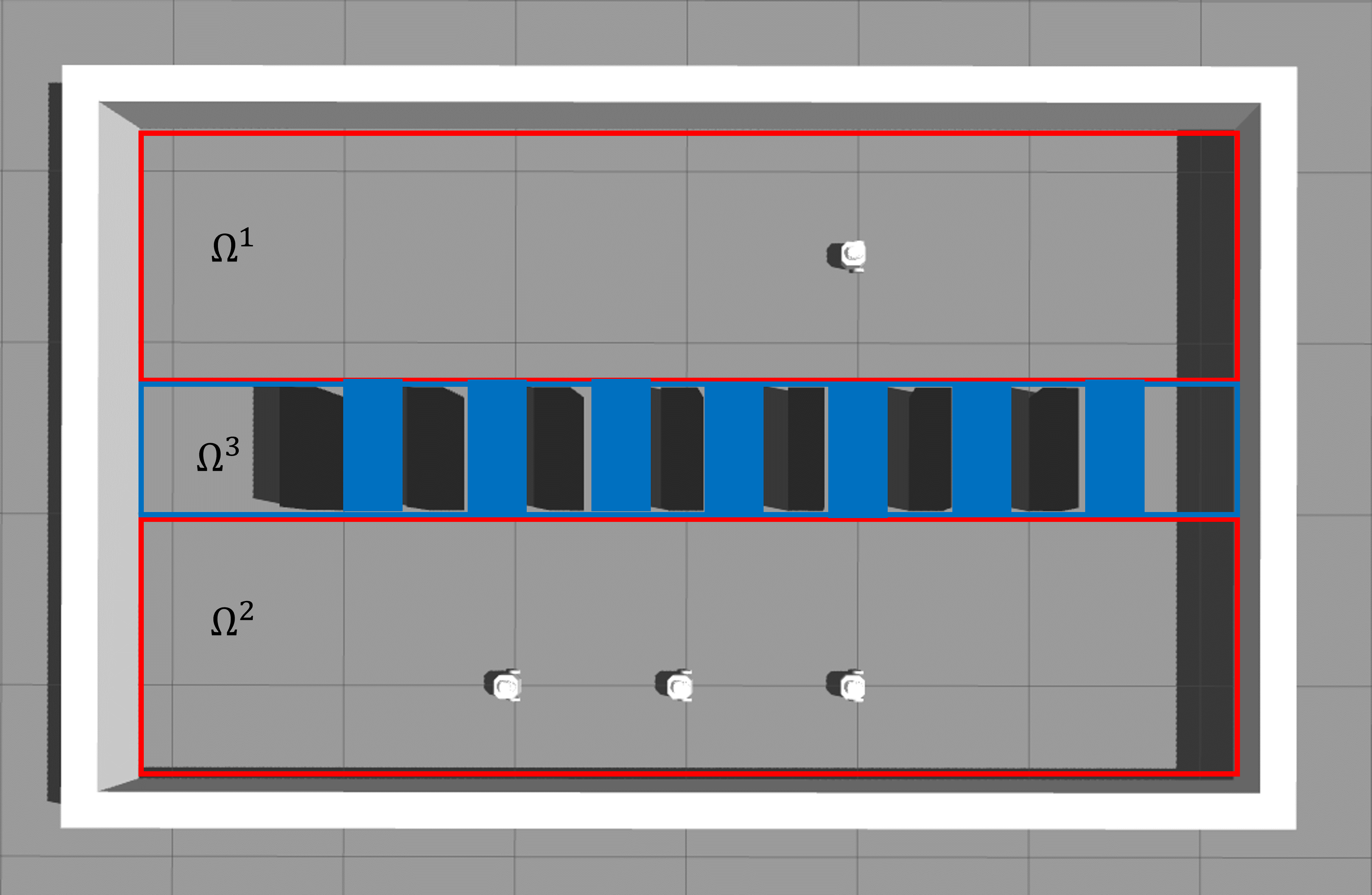}
    \caption{Regions in arena. The lanes are shown by vertical blue shaded area.}
    \label{fig:proof_piano_regions}
\end{figure}
let us divide the arena into three regions, $\Omega^1$ (a region above the lanes) ,$\Omega^2$  (a region below the lanes), and $\Omega^3$ (remaining region of lanes),  as shown in Fig \ref{fig:proof_piano_regions}.  

Without loss of generality we assume that a  robot $R_i$ is initialized in $\Omega^1$. As robot is in $\Omega^1$, it will be asked to visit entry location of lane which lies in $\Omega^1$. 
The region  $\Omega^1$ is convex,  using the Theorem \ref{theorem:2}, robot  will have collision-free trajectory. 

For playing music but $R_i$ have to wait at the entry point till the desired time and them posses through lanes i.e. $\Omega^3$  and goes to region $\Omega^2$. 
From region  $\Omega^1$,  $R_i$ can enter region  $\Omega^3$ only for playing music and only  at desired task time.  Only one robot is assigned to play a musical note, so no two robots will enters the lane at same time. Hence robots won't collide with any other robots in $\Omega^3$.

Once robot reaches the other end of lane, it enters in  the region $\Omega^2$. For executing the next task robot will be asked to visit the entry point of assigned lane which lies in $\Omega^2$. Region $\Omega^2$ is convex and hence the robot will have collision-free trajectory in $\Omega^2$.

Robot will have collision-free trajectory whenever it is in region $\Omega^1$,$\Omega^2$, and $\Omega^3$. Hence, robot has collision-free trajectory in entire arena.
\end{proof}

\section{Simulation and Results}\label{sec:Simulations}
\subsection{Simulation  Architecture  }
The DREAM algorithm has been used to play piano music. This system is implemented and tested with the help of ROS2 in the Gazebo simulation environment. The simulations are conducted in Ubuntu 20.04 system with a 3.20GHz processor and 16 GB RAM.
An open-source Nav2 package \cite{macenski2020marathon} is used used to navigate the multiple robots in the environment. It is also used to compute a feasible path from one location to another. The distance between the two points is computed using the line integral along the computed feasible path.

\subsubsection{ROS topics, services and actions}
The communication between robots works depends on topics, services, and actions. Topic-based communications happen in continuous publishing and subscribing fashion. The values or data are published by one robot asynchronously. IMU readings and poses are published to respective topics; the same topics are subscribed by another robot or process and vice-versa. This enables the algorithm to receive pose information and command actions to navigate the robots. 
\begin{figure}[t!]
    \centering
    \includegraphics[width=0.8\linewidth]{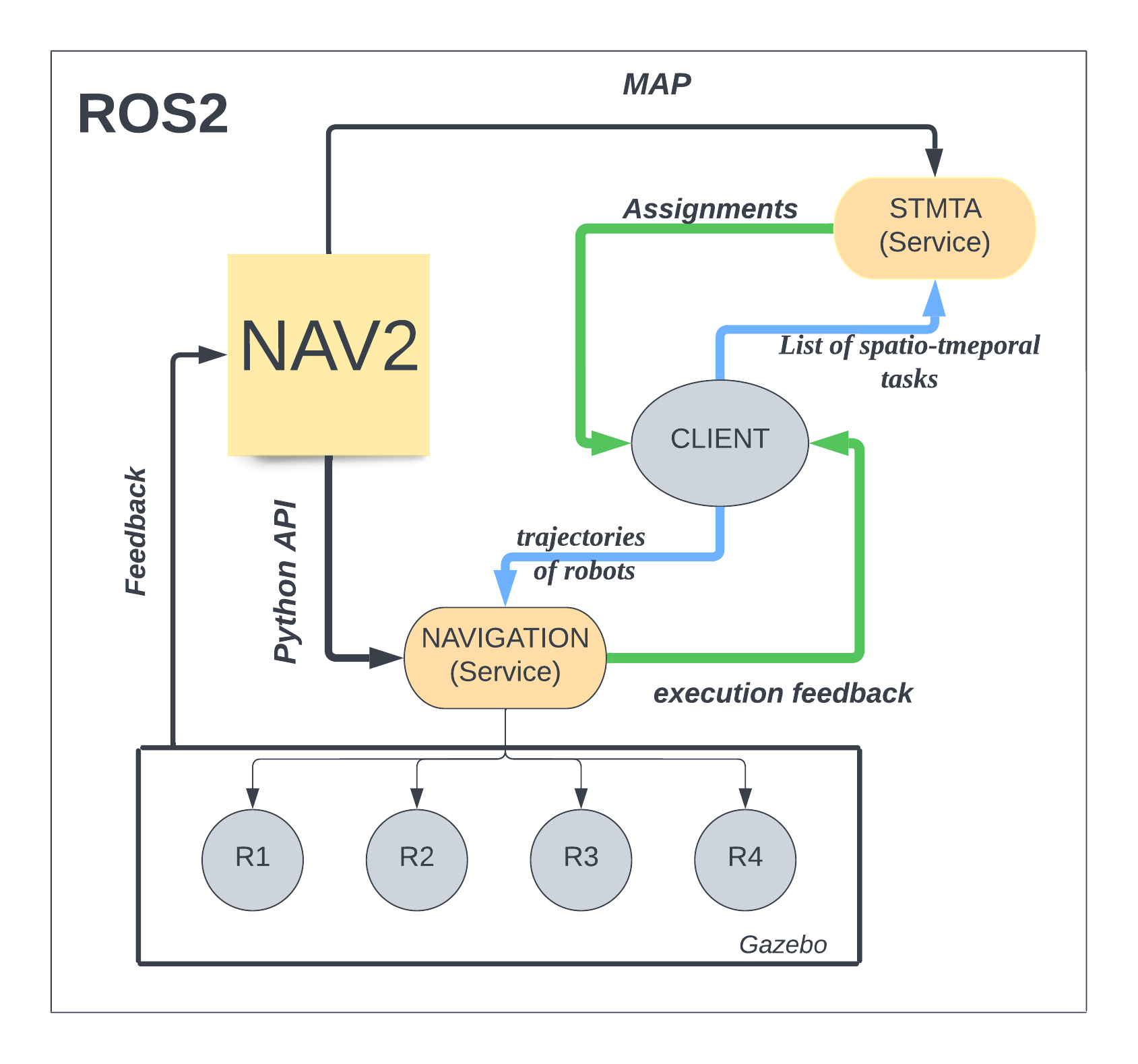}
    \caption{Simulation architecture}
    \label{fig:architecture}
\end{figure}

Services are another way to pass data between nodes in ROS2. They are synchronous or blocking remote procedure calls where one node (client) can call a function that executes in another node (service). Fig \ref{fig:architecture} shows the simulation architecture consisting of navigation and spatio-temporal multi-task assignment (STMTA) services.
For the computation of the cost, i.e., for computing the distance from one point to another (on the live map), the 'ComputePathToPose' service from the 'nav2' package has been used. 
Once the cost matrices are computed, the  STMTA service is called to compute tasks assigned to robots. This service executes the DREAM algorithm and provides feasible and optimal task assignments. It first computes the required number of robots and then, with an updated team of robots, computes the assignments such that all musical notes are played at respective times while minimizing the total distance traveled.

After computing the assignments, each robot computes its trajectory (sequence of tasks). Then multi-robot navigation service is invoked to operate all the robots asynchronously. Each robot executes its own sequence of tasks; it reaches the waiting area near the assigned note, waits till the desired time, and then passes through the lane of note to play the musical note. Each task is executed using the asynchronous navigation service for the simultaneous operations of robots.

\subsection{ Performance Evaluation  } \label{sec:results}
The Piano setup is shown in Fig. \ref{fig:initail_conditoin}. 
The team of 4 robots plays a happy birthday tune in the piano arena to illustrate the working of the proposed approach. A total of 7 different musical notes are used. Here, the gap between the two walls is 0.5m, and the lane length is 0.4m, but for turning robots (considering the waiting area), the practical distance needs to be traveled is 0.8 m. The turtlebots (with a max velocity of 0.5 m/sec) are used as robots to travel in the arena and play music. As the distance that needs to be traveled is large compared to the velocity of the robots, the musical tasks are given for execution by scaling the time ten times. 

\begin{figure}[bhtp!]
    \centering
      \vspace{10pt}
    \includegraphics[width=0.8\linewidth]{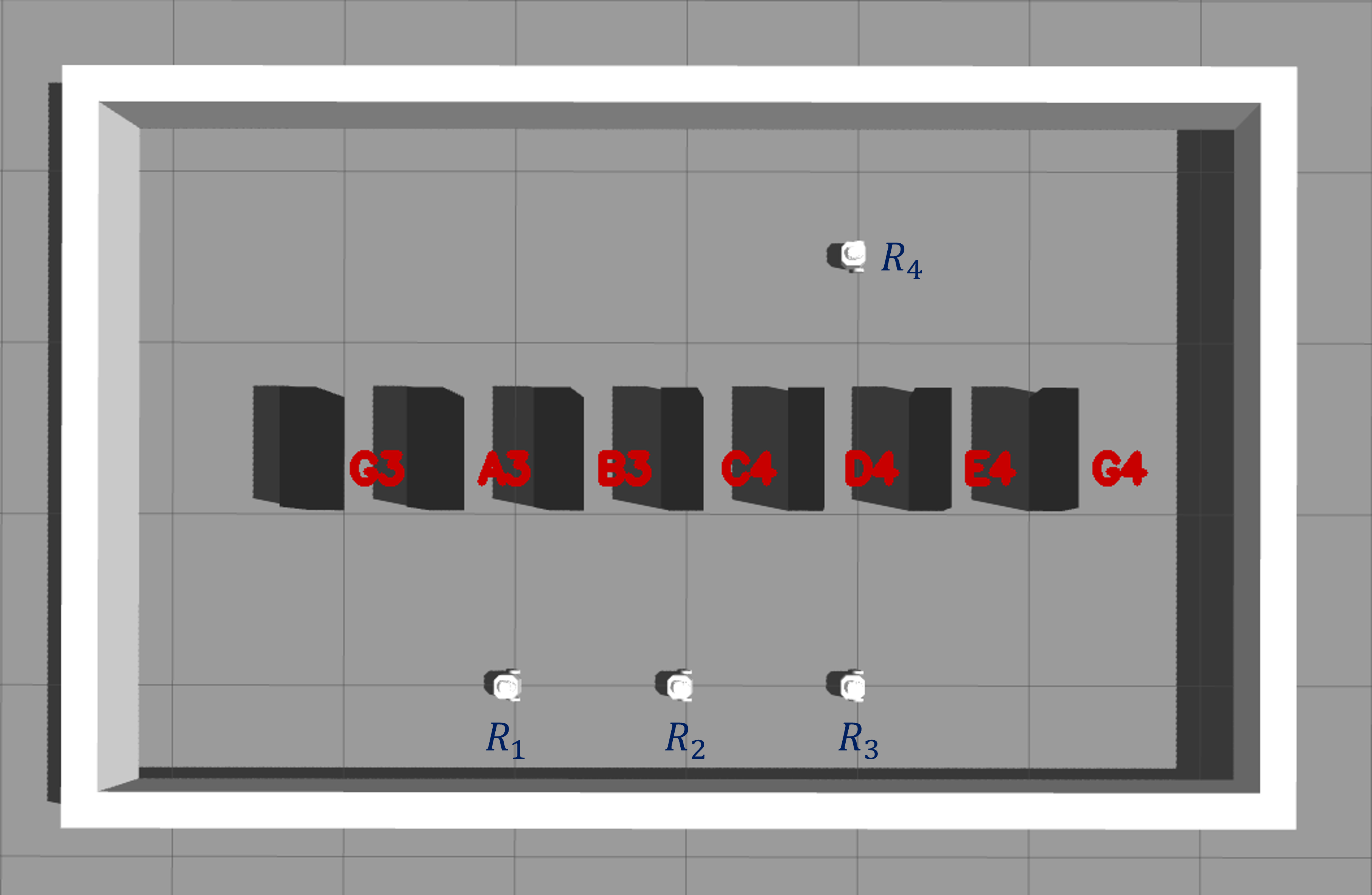}
    \caption{Piano set-up and initial conditions of robots}
    \label{fig:initail_conditoin}
\end{figure}

A total of 24 musical notes are required to play the tune and the spatio-temporal tasks considered are given in table \ref{table:tasks}. The table gives the musical note to be played and its desired time of play. Each note can be played by traveling through a specific lane on the piano, and these lanes and their respective notes are shown in Fig \ref{fig:initail_conditoin}. The  attached video demonstrates the music-playing robots.

\begin{table}[htbp!]
\vspace{10pt}
\centering
\setlength{\arrayrulewidth}{0.5mm}
\setlength{\tabcolsep}{8pt}
\renewcommand{\arraystretch}{1.1}
\begin{tabular}{ |c|c|r|    }
\hline
\begin{tabular}[c]{@{}c@{}}\textbf{Task} \\ \textbf{No}\end{tabular} & \textbf{Note} & \begin{tabular}[c]{@{}c@{}}\textbf{time }\\  \textbf{(sec)}\end{tabular} \\  
\hline
1 &G3 &105\\
\hline
2 &G3 &113  \\
\hline
3 &A3 &119 \\
\hline
4 &G3 &124\\
\hline
5 &C4 &133 \\ 
\hline
6 &B3 &142\\
\hline
7 &G3 &157\\
\hline
8 &G3 &163\\
\hline
9 &A3 &168\\
\hline
10 &G3 &172\\
\hline
11 &D4 &183\\
\hline
12 &C4 &192\\
\hline
\end{tabular} \quad%
\begin{tabular}{|c|c|r|}
\hline
\begin{tabular}[c]{@{}c@{}}\textbf{Task}\\\textbf{No}\end{tabular} & \textbf{Note} & \begin{tabular}[c]{@{}c@{}}\textbf{time }\\  \textbf{(sec)}\end{tabular} \\   
\hline
13 &G3 &206\\
\hline
14 &G3 &214\\
\hline
15 &G4 &221\\
\hline
16 &E4 &229\\
\hline
17 &C4 &235\\
\hline
18 &B3 &239\\
\hline
19 &A3 &250\\
\hline
20 &G4 &260\\
\hline
21 &G4&270 \\
\hline
22 &C4 &280\\
\hline
23 &D4 &290\\
\hline
24 &D4 &300\\
\hline
\end{tabular}
\caption{Spatio-temporal tasks for playing happy birthday tune}
    \label{table:tasks}
\end{table}

Once all tasks are received, the robots use the map of the piano arena to compute the distances and then compute the cost matrices. When only one robot is used to solve STMTA, it demands three extra robots. So, total of four robots has been used to play the desired tune.
The computed trajectories for the robots are,
for robot 1 , $\mu_1 = 
\{ T_{1},T_{2},T_{4},T_{7},T_{8},T_{10},T_{13},T_{14},T_{19} \}$;
for robot 2, $\mu_2 = \{ T_{3}, T_{9},T_{12},T_{17},T_{22},T_{23},T_{24} \}$; for robot 3, $\mu_3 = \{ T_{5},T_{6},T_{18}  \}$,
for robot 4, $\mu_4 = \{ T_{11},T_{15},T_{16},T_{20},T_{21}  \}$.

\begin{figure}[thbp!]
    \centering
    \includegraphics[width=0.9\linewidth]{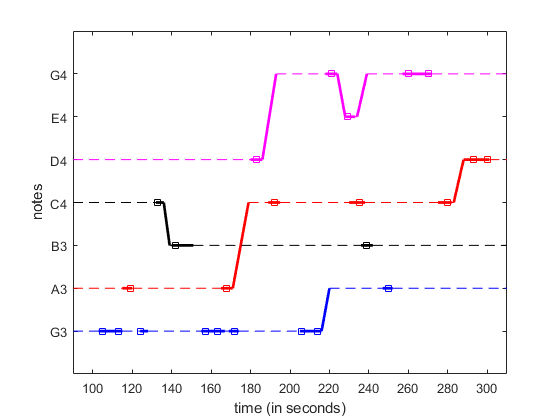}
    \caption{Operations of robots while playing music}
    \label{fig:result_execution}
\end{figure}
The attached video shows the operations of robots to play the happy birthday tune. Fig \ref{fig:result_execution} shows the operations of robots. Four different colors are used for the four robots. The solid line denotes the robots in motion, and the dashed lines denote the robots waiting. One can observe that four robots are required to execute tasks during the time interval of $230$ sec to $250$ sec. If any robot is absent, then all notes can not be played. It shows that the team size computed by the DREAM (i.e., four robots in this case) is a necessary and sufficient size of the team to execute the given spatio-temporal tasks.

\section{Conclusions} \label{sec:Conclusion}
In this paper, we considered the spatio-temporal multi-task assignment problem for playing piano music using a team of robots. The dynamic resource allocation with multi-task assignments (DREAM) approach can be directly used to compute the minimum team size and the optimal assignments (which minimizes the total distance traveled). The DREAM approach solves the bottleneck issue of iterative computation for the required number of robots and provides the two-step solution to compute the required minimum number of robots and their optimal assignments to execute given spatio-temporal tasks. This paper analyses the DREAM algorithm and proves that the solution gives collision-free trajectories for all robots when operated in convex regions. The working of the DREAM approach is demonstrated for the piano music-playing robot simulations in a ROS2-Gazebo environment. Future work focuses on the heterogeneous robots playing music on more complex and various instruments and exploitation of guaranteed collision-free trajectories for various STMTA problems.

%
\section*{ACKNOWLEDGMENT}
Authors would like to acknowledge the financial support provided by Nokia CSR funds and Nokia Network Robotics laboratory at IISc.

\bibliographystyle{IEEEtran}
\bibliography{main_bib.bib}
\end{document}